\documentclass[letterpaper, 10pt, conference]{icra_ieeeconf}  

\IEEEoverridecommandlockouts                              

\usepackage{multirow}
\usepackage{hyperref}
\usepackage{booktabs}
\usepackage{amsmath} 
\usepackage{amssymb}  
\usepackage{amsfonts}  
\usepackage{nicematrix}
\usepackage{subfig,graphicx}
\usepackage{textcomp}
\usepackage{siunitx}
\usepackage{xcolor}
\usepackage{bm}
\usepackage{diagbox}
\usepackage{algorithm}
\usepackage{algpseudocode}
\usepackage{cite}

\newtheorem{theorem}{Theorem}
\newtheorem{definition}{Definition}

\newtheorem{lemma}{Lemma}

\newcommand{\lW}{ {}^{\mathcal{W}} }
\newcommand{\lK}{ {}^{\mathcal{K}} }

\newcommand{\lCK}{ {}^{\mathcal{C}}_{\mathcal{K}} }

\newcommand{\lCI}{ {}^{\mathcal{C}}_{\mathcal{I}} }

\newcommand{\vect}[1]{\bm{#1}}
\newcommand{\mat}[1]{\mathbf{#1}}

\graphicspath{{./Images/}}
\begin{document}

\title{\LARGE \bf
GeVI-SLAM: Gravity-Enhanced Stereo Visual–Inertial SLAM for Underwater Robots}
\author{%
Yuan Shen$^{1,*}$, %
Yuze Hong$^{1,*}$, %
Guangyang Zeng$^{1,\dagger}$, %
Tengfei Zhang$^{1}$, %
Pui Yi Chui$^{2}$, %
Ziyang Hong$^{1}$, %
Junfeng Wu$^{1}$%
\thanks{$^{1}$School of Data Science, Chinese University of Hong Kong, Shenzhen, China. {\tt\footnotesize junfengwu@cuhk.edu.cn}}%
\thanks{$^{2}$School of Life Sciences, Chinese University of Hong Kong, China.}%
\thanks{$^{\dagger}$Corresponding author: {\tt\footnotesize zengguangyang@cuhk.edu.cn}}%
\thanks{$^{*}$These authors contributed equally to this work.}%
}

\maketitle
\thispagestyle{empty}
\pagestyle{empty}

\begin{abstract}
Accurate visual–inertial simultaneous localization and mapping (VI SLAM) for underwater robots remains a significant challenge due to frequent visual degeneracy and insufficient inertial measurement unit (IMU) motion excitation. 
In this paper, we present GeVI-SLAM, a gravity-enhanced stereo VI SLAM system designed to address these issues. By leveraging the stereo camera's direct depth estimation ability, we eliminate the need to estimate scale during IMU initialization, enabling stable operation even under low-acceleration dynamics.
With precise gravity initialization, we decouple the pitch and roll from the pose estimation and solve a 4 degrees of freedom (DOF) Perspective‑n‑Point (PnP) problem for pose tracking. This allows the use of a minimal 3-point solver, which significantly reduces computational time to reject outliers within a Random Sample Consensus framework.
We further propose a bias-eliminated 4-DOF PnP estimator with provable consistency, ensuring the relative pose converges to the true value as the feature number increases. 
To handle dynamic motion, we refine the full 6-DOF pose while jointly estimating the IMU covariance, enabling adaptive weighting of the gravity prior.
Extensive experiments on simulated and real-world data demonstrate that GeVI-SLAM achieves higher accuracy and greater stability compared to state-of-the-art methods.
\end{abstract}

\section{INTRODUCTION}
Underwater robotic perception challenges conventional algorithms\cite{mallios2017underwater,joshi2019experimental}. For visual–inertial (VI) systems, beyond well-documented issues like poor illumination and turbidity\cite{xu2025aqua}, underwater scenes introduce a set of compounding difficulties that are detrimental to perception performance. 

First, natural terrains and man-made structures 
often exhibit repetitive or weak textures, inflating feature ambiguity and yielding high outlier ratios in feature matching and tracking \cite{rahman2019svin2}; a frontend must therefore be substantially more robust than in typical terrestrial scenes. Second, scenes are often sparse and near-planar (e.g., flat seabeds and bridge piers), which results in geometric degeneracy and leads to unstable or incorrect state estimation~\cite{yao2024sg}. Third, low accelerations due to high underwater resistance undermine system initialization: VINS-Mono \cite{qin2018vins} suffers a scale--gravity--bias ambiguity under near-constant velocity \cite{hu2021visual}; the stereo--inertial measurement unit (IMU) mode in ORB-SLAM3\cite{campos2021orb} can recover scale but lacks excitation to accurately estimate IMU biases and gravity.
\begin{figure}[!t]
    \centering
    \includegraphics[width=\columnwidth]{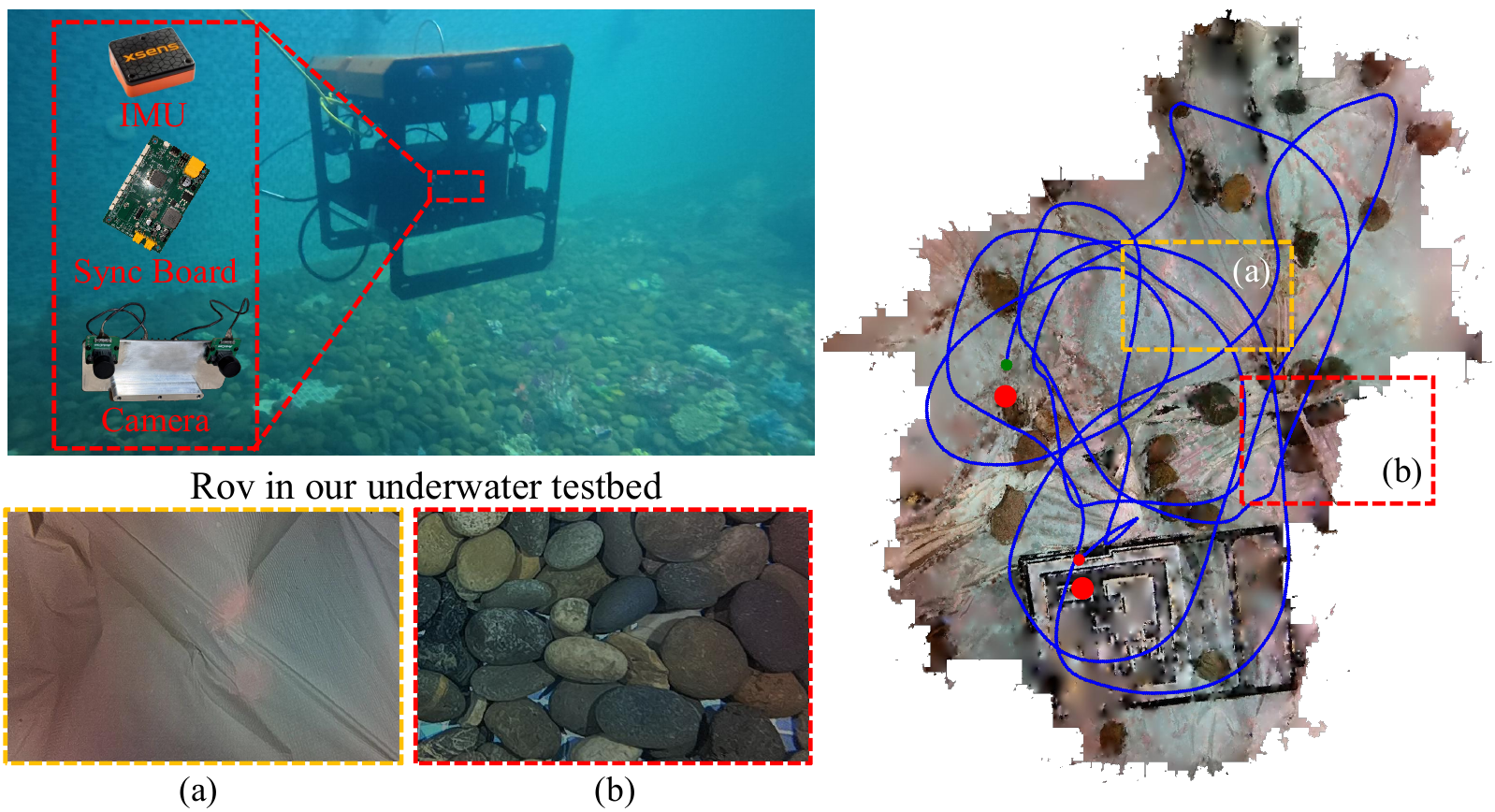}
    \caption{
        Estimated trajectory (blue line) using the proposed GeVI-SLAM system, visualized on a dense 3D reconstruction. It demonstrates our algorithm's robustness against common underwater challenges: (a) \textbf{feature-sparse, coplanar scenes} and (b) \textbf{highly repetitive patterns}.
    }
    \label{fig:cover_figure}
\end{figure}
\textbf{Contributions.} To tackle the above challenges, we propose a gravity-enhanced underwater VI SLAM system. We note that although low accelerations hinder initialization yet offer a stable and precise gravity prior once the gravity has been well initialized. Hence, we first design a stereo–IMU initialization with known scale, minimizing preintegration residuals to robustly estimate gravity and IMU biases. Then, our main focus is to leverage IMU gravity prior to decouple 6-degrees of freedom (DOF) pose and build a 4-DOF (yaw angle and 3D translation) visual frontend, which is resilient to high outlier rates and geometric degeneracy.
Furthermore, adaptive VI fusion gives rise to accurate full 6-DOF refinement.
In summary, our contributions are three-fold:
\begin{itemize}
    \item \textbf{Gravity-enhanced 4-DOF Perspective‑n‑Point (PnP):} By using gravity prior to reduce pose to 4-DOF, a minimal 3-point Random Sample Consensus (RANSAC) delivers faster and more robust outlier rejection (our required iteration number, relative to 5-point RANSAC, is on the order of the squared inlier ratio). In addition, our 4-DOF PnP estimator is proven to be consistent and performs stably in near-degenerate scenes.

    \item \textbf{Adaptive VI fusion:} We adopt a block coordinate descent (BCD) strategy to jointly refine 6-DOF pose and estimate gravity prior covariance, enabling adaptive weighting of the gravity constraint and effectively avoiding the drift of pitch and roll angles.    

    \item \textbf{Comprehensive experiment validation:} Simulations demonstrate Cramér-Rao Lower Bound (CRLB)-tight accuracy, robust minimal-set consensus under up to 30\% outliers, and drift-free pitch/roll estimation. Extensive real water pool experiments show consistently lower Absolute Trajectory Error (ATE) and Relative Pose Error (RPE) than ORB-SLAM3~\cite{campos2021orb}, VINS-Fusion~\cite{qin2018vins}, and SVIN2~\cite{rahman2019svin2}, and stable tracking where these baselines may fail. 
\end{itemize}

\textbf{Notations:} 
Scalars use italic lowercase ($a$), vectors bold lowercase ($\mathbf{a}$), and matrices bold uppercase ($\mathbf{A}$); unless stated otherwise, vectors are columns. Estimated quantities are marked with a hat $\hat{\mathbf{x}}$, while true values have a superscript circle $\mathbf{x}^o$. 
Reference frames are denoted by calligraphic letters: $\mathcal{W}$ (world), $\mathcal{C}$ (camera), $\mathcal{I}$ (IMU), and $\mathcal{K}$ (keyframe). 
A rigid body transformation from frame $\mathcal{A}$ to $\mathcal{B}$ is denoted by ${}^{\mathcal{B}}_{\mathcal{A}}\mathbf{T} \triangleq [{}^{\mathcal{B}}_{\mathcal{A}}\mathbf{R} | {}^{\mathcal{B}}_{\mathcal{A}}\mathbf{t}] \in \mathrm{SE}(3)$, where ${}^{\mathcal{B}}_{\mathcal{A}}\mathbf{R} \in \mathrm{SO}(3)$ is the rotation and ${}^{\mathcal{B}}_{\mathcal{A}}\mathbf{t} \in \mathbb{R}^3$ is the translation. A point ${}^{\mathcal{A}}\mathbf{p}$ is transformed to frame $\mathcal{B}$ via ${}^{\mathcal{B}}\mathbf{p} = {}^{\mathcal{B}}_{\mathcal{A}}\mathbf{R} {}^{\mathcal{A}}\mathbf{p} + {}^{\mathcal{B}}_{\mathcal{A}}\mathbf{t}$. 
The world frame $\mathcal{W}$ is defined such that its $z$-axis is aligned with the gravity vector.

\section{Related Work}
\subsection{Underwater Visual SLAM}
Underwater VI SLAM can be broadly categorized into two groups. Feature-based methods match sparse features and estimate motion via epipolar geometry or PnP; they are robust to viewpoint and illumination changes but can be computationally heavy and vulnerable in low-texture underwater scenes~\cite{hildebrandt2010imu,bellavia2017selective}. In contrast, direct methods minimize photometric error over image intensities and are more efficient and tolerant to weak texture; however, they require good initialization and are sensitive to illumination changes~\cite{yang2017feature,quattrini2016experimental}. Recent tightly coupled underwater VI SLAM systems mainly adopt feature-based frontends fused with IMU in joint optimization. For example, Wang \emph{et al.}~\cite{wang2023robust} fused ORB points, diagonal segments, IMU preintegration, and loop closure to achieve low Root Mean Square Error (RMSE) in real time. Meanwhile, unified variants take advantage of the two methods, e.g., UniVIO~\cite{miao2021univio} jointly optimized feature-based reprojection and direct photometric errors together with inertial terms in a unified stereo-IMU backend. 
Finally, some tightly coupled multi-sensor estimators combine cameras and IMU with more sensors, such as acoustic and depth sensors~\cite{rahman2019svin2,huang2023tightly}. 

\subsection{Leveraging Gravity Priors in SLAM}
Gravity provides a global prior fixing roll and pitch, leaving yaw as the only rotational DOF. Lupton \emph{et al.}~\cite{lupton2011visual} preintegrated IMU measurements to recover gravity and velocity, enabling constant-time VI SLAM. Li \emph{et al.}~\cite{li2019rapid} refined gravity on a unit sphere using vertical edges to decouple accelerometer bias and speed up initialization. Usenko \emph{et al.}~\cite{usenko2016direct} tightly coupled IMU with a stereo camera such that roll and pitch are observable and alignment is robust under fast motion and low texture. GaRLIO~\cite{noh2025garlio} fused radar Doppler with LiDAR and IMU for a velocity-aware gravity estimate that suppresses vertical drift. Kubelka \emph{et al.}~\cite{kubelka2022gravity} enforced gravity in LiDAR iterative closest point by fixing roll and pitch and solving only yaw and translation, reducing median normalized ATE by around 30\%.

The existing work mainly uses gravity as a cost term in iterative optimization. In this paper, we not only incorporate the gravity to achieve reduced-degree visual tracking, which allows a minimal-set closed-form solution, but also adaptively weight the gravity prior to account for time-varying motion acceleration.

\section{Gravity-Enhanced 4-DOF camera pose estimation}\label{section::4_dof_est}

\subsection{Euler Angles and Characterization of a 4-DOF Pose}
Unless otherwise stated, a camera frame refers to the left camera frame of a stereo camera. We aim to estimate the relative pose $\lCK\mathbf{T} \in \mathrm{SE}(3)$ from the camera keyframe $\mathcal{K}$ to the current camera frame $\mathcal{C}$. This transformation is composed of a rotation matrix $\lCK \mathbf{R} \in \mathrm{SO}(3)$ and a translation vector $\lCK \mathbf{t} \in \mathbb{R}^3$. The rotation $\lCK \mathbf{R}$ can be parameterized using three Euler angles: roll ($\phi$), pitch ($\theta$), and yaw ($\psi$). For our application, we decompose the rotation into two components: a rotation representing yaw, $\mathbf{R}_\psi$, and a rotation representing the combined roll and pitch, $\mathbf{R}_{\theta\phi}$. The full rotation is expressed as $\lCK \mathbf{R} = \mathbf{R}_\psi \mathbf{R}_{\theta\phi}$,
where $\mathbf{R}_\psi=\left[\begin{smallmatrix}
    \cos{\psi} & \sin{\psi} & 0\\
    -\sin{\psi} & \cos{\psi}  & 0\\
    0          & 0           & 1
\end{smallmatrix}\right]$
and $\mathbf{R}_{\theta\phi}=\left[\begin{smallmatrix}
   \cos{\theta}  & -\sin{\theta}\sin{\phi} & \sin{\theta}\cos{\phi}\\
    0 & \cos{\phi}  & \sin{\phi}\\
    -\sin{\theta} & -\cos{\theta}\sin{\phi} & \cos{\theta}\cos{\phi}
\end{smallmatrix}\right]$.

This factorization is particularly useful because in a world frame $\mathcal{W}$, the $xy$-plane is aligned with the ground, and an IMU can estimate roll and pitch relative to gravity~\cite{mahony2008nonlinear}. With roll and pitch known, only yaw remains, motivating a 4-DOF pose consisting of the yaw angle $\psi$ and $\lCK \mathbf{t}$.

\subsection{Acquisition of Drift-Free Roll and Pitch }\label{section::imu_gravity_est}
To provide a reliable roll and pitch for our 4-DOF PnP solver, we fuse measurements from the IMU's gyroscope and accelerometer. Our approach is inspired by the complementary filtering\cite{mahony2008nonlinear}, which optimally combines the strengths of both sensors. Gyroscopes yield accurate high-frequency attitude updates over short intervals but drift over time. Accelerometers provide a drift-free, low-frequency gravity reference in quasi-static motion (stationary or near-constant velocity) but are sensitive to linear acceleration. By integrating gyroscope measurements for dynamic tracking and using accelerometer data within a probabilistic optimization framework (detailed in Section~\ref{jointly_est_refine}) for drift correction, our system can maintain accurate pitch and roll angles.

\subsection{Problem Formulation}\label{subsec:4dof_problem}


As illustrated in Fig.~\ref{fig:epipolar_model}, after feature tracking and outlier removal, we obtain 2D point correspondences $\{{{\bf q}_i, {\bf z}_i, {\bf y}_i,}\}_{i \in \mathcal I}$, where $\mathcal I$ is the inlier set, ${\bf q}_i$'s denote the points in the left image of the current frame, and ${\bf z}_i$'s and ${\bf y}_i$'s are the points in the left and right images of the keyframe, respectively. Unless otherwise stated, image 2D points are expressed using the normalized image coordinates. 
We assume the feature tracking errors are independent and identically distributed (i.i.d.) zero-mean Gaussian noises, i.e., $\mathbf{z}_i \sim \mathcal{N}(\mathbf{z}_i^o, \sigma^2 \mathbf{I}_2)$ and $\mathbf{y}_i \sim \mathcal{N}(\mathbf{y}_i^o, \sigma^2 \mathbf{I}_2)$. The zero-mean and isotropic covariance is reasonable because, after outlier rejection, the errors from intensity-based trackers tend to be unbiased and inherit the spatially uniform nature of physical sensor noise. This model is not only a good approximation for pinhole cameras~\cite{EIVPNPRAL2025} but also practical, as its variance $\sigma^2$ can be consistently estimated from data~\cite{zeng2024consistent}.


Leveraging the 4-DOF parameterization introduced previously, our goal is to estimate the state vector $\boldsymbol{\chi} = [\psi, ~\lCK \mathbf{t}^\top]^\top$. From this state, the full relative rotation is constructed as $\lCK \mathbf{R}(\psi) = \mathbf{R}_\psi \mathbf{R}_{\theta\phi}$, where \(\mathbf{R}_{\theta\phi}\) is the roll and pitch estimate obtained from IMU information. Here, we utilize point-to-epipolar-line distances as costs and construct a maximum likelihood (ML) estimation problem. For the $i$-th correspondence, we formulate two geometric costs. The first relates to the current frame observation $\mathbf{q}_i$ and the keyframe's left image observation $\mathbf{z}_i$. 
The epipolar line $\mathbf{l}_{i, \mathcal{K}_L}$ in the left keyframe image is generated by back-projecting $\mathbf{q}_i$ using the pose ${}^{\mathcal{C}}_{\mathcal{K}}\mathbf{T} = [{}^{\mathcal{C}}_{\mathcal{K}}\mathbf{R}(\psi) ~| ~{}^{\mathcal{C}}_{\mathcal{K}}\mathbf{t}]$: $\mathbf{l}_{i, \mathcal{K}_L} = \left([{}^{\mathcal{C}}_{\mathcal{K}}\mathbf{t}]_{\times} {}^{\mathcal{C}}_{\mathcal{K}}\mathbf{R}(\psi) \right)^\top \mathbf{q}_i^h$,
where $\mathbf{q}_i^h$ is the homogeneous coordinates of $\mathbf{q}_i$, and $[\cdot]_\times$ is the operator that maps a vector in $\mathbb{R}^3$ to its corresponding $3 \times 3$ skew-symmetric matrix, such that for any two vectors $\mathbf{a}, \mathbf{b} \in \mathbb{R}^3$, $[\mathbf{a}]_\times \mathbf{b} = \mathbf{a} \times \mathbf{b}$. Then, the distance of point ${\bf z}_i$ to epipolar line $\mathbf{l}_{i, \mathcal{K}_L}$ is $d_{i, \mathcal{K}_L}=\frac{\mathbf{l}_{i, \mathcal{K}_L}^\top \mathbf{z}_i^h}{\| \mathbf{E} \mathbf{l}_{i, \mathcal{K}_L} \|}$,
where $\mathbf{E} = [{\bf e}_1,~{\bf e}_2]^\top$, and ${\bf e}_i$ denotes the unit vector in $\mathbb R^3$ whose $i$-th element is 1. The second cost, defined in the keyframe’s right image, uses the pose ${}^{\mathcal{C}}_{\mathcal{K}_R}\mathbf{T}$ obtained via the stereo extrinsic:
${}^{\mathcal{C}}_{\mathcal{K}_R}\mathbf{T} = {}^{\mathcal{C}}_{\mathcal{K}}\mathbf{T}\,({}^{\mathcal{K}_R}_{\mathcal{K}}\mathbf{T})^{-1} = [\,{}^{\mathcal{C}}_{\mathcal{K}_R}\mathbf{R}\;|\;{}^{\mathcal{C}}_{\mathcal{K}_R}\mathbf{t}\,]$.
The corresponding epipolar line is
$\mathbf{l}_{i,\mathcal{K}_R} = \left([{}^{\mathcal{C}}_{\mathcal{K}_R}\mathbf{t}]_{\times}\,{}^{\mathcal{C}}_{\mathcal{K}_R}\mathbf{R}\right)^\top \mathbf{q}_i^h$,
and the distance is
$d_{i,\mathcal{K}_R}=\frac{\mathbf{l}_{i,\mathcal{K}_R}^\top \mathbf{y}_i^h}{\| \mathbf{E}\,\mathbf{l}_{i,\mathcal{K}_R} \|}$. 


Under the aforementioned noise assumptions on $\mathbf{z}_i$ and $\mathbf{y}_i$, one can verify that the resulting point-to-epipolar-line distance errors follow the Gaussian distribution $\mathcal N(0,\sigma^2)$. Hence, the ML estimate for the 4-DOF relative pose $\boldsymbol{\chi}$ is obtained by solving the following non-linear least-squares problem over all inlier correspondences $\mathcal{I}$:
\begin{equation} \label{eq:4dof_objective}
    \underset{\boldsymbol{\chi} \in \mathbb R^4}{\min} \sum_{i \in \mathcal{I}} \left( d_{i, \mathcal{K}_L}^2 + d_{i, \mathcal{K}_R}^2 \right).
\end{equation}
The optimal solution to the ML problem~\eqref{eq:4dof_objective} is called the ML estimate, denoted as $\hat {\bm \chi}^{\rm ML}$. Solving this non-linear optimization is challenging as the objective function is non-convex, making iterative solvers highly dependent on a good initial guess. To circumvent this issue, we propose linearizing the model to derive an efficient and consistent initial estimator.


\subsection{Consistent 4-DOF Estimator}\label{subsec:consistent_init}
To initialize the non-linear optimization~\eqref{eq:4dof_objective}, we develop a consistent estimator for the 4-DOF relative pose. The process involves three main steps: uncertainty-aware 3D point triangulation, construction of a linear model, and a closed-form bias-eliminated (BE) estimator.
\begin{figure}[h]
    \centering
    \includegraphics[width=0.85\linewidth]{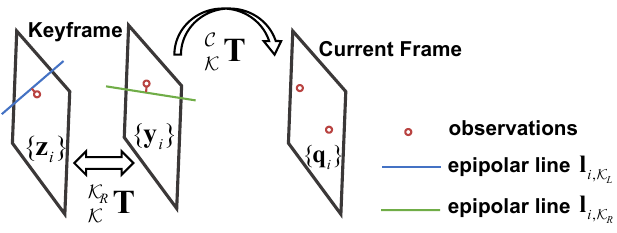}
    \caption{4-DOF pose estimation based on point-to-epipolar-line distances.}
    \label{fig:epipolar_model}
\end{figure}
\subsubsection{Uncertainty-Aware 3D Point Triangulation}
First, we use the method in~\cite{zeng2024consistent} to obtain a consistent estimate $\hat{\sigma}^2$ of 2D noise variance $\sigma^2$. It involves solving a generalized eigenvalue problem. For the details, one can refer to Appendix A of~\cite{zeng2024consistent}.  
Then, following equations (1) and (2) in~\cite{EIVPNPRAL2025}, we can triangulate the 3D points $\lK \mathbf{p}_i$'s from stereo observations in the previous keyframe along with their uncertainties $\boldsymbol{\Sigma}_{\mathbf{p}_i}$'s. As a result, we establish 3D-2D point correspondences $\{\lK \mathbf{p}_i,{\bf q}_i\}_{i \in \mathcal I}$, where the 3D points $\lW \mathbf{p}_i$'s are noisy with known covariance $\boldsymbol{\Sigma}_{\mathbf{p}_i}$, while their corresponding 2D matches $\mathbf{q}_i$'s in the current frame are noise-free.

\subsubsection{Model Linearization}
Our goal is to linearize the camera projection model to enable an efficient closed-form solution. We begin with the ideal projection of a true 3D point ${^{\mathcal{K}}}\mathbf{p}_i^o$, expressed in the left keyframe, into the current frame's normalized image plane:
\begin{equation}\label{eq:ideal_proj_rephrased}
    \mathbf{q}_i = h( {^{\mathcal{K}}}\mathbf{p}_i^o; \boldsymbol{\chi}) \triangleq \frac{\mathbf{E}\left(\lCK \mathbf{R}(\psi) {^{\mathcal{K}}}\mathbf{p}_i^o + \lCK\mathbf{t}\right)}{{\bf e}_3^\top(\lCK \mathbf{R}(\psi) {^{\mathcal{K}}}\mathbf{p}_i^o + \lCK\mathbf{t})},
\end{equation}
where $h$ is the pinhole camera projection function. Recall that the 2D point $\mathbf{q}_i$ is noise-free, and the triangulated 3D point ${^{\mathcal{K}}}\mathbf{p}_i$ is noisy, modeled as ${^{\mathcal{K}}}\mathbf{p}_i = {^{\mathcal{K}}}\mathbf{p}_i^o + \boldsymbol{\epsilon}_{\mathbf{p}_i}$, with $\boldsymbol{\epsilon}_{\mathbf{p}_i} \sim \mathcal{N}(\mathbf{0}, \boldsymbol{\Sigma}_{\mathbf{p}_i})$.

To linearize the model, we multiply both sides of equation~\eqref{eq:ideal_proj_rephrased} by ${\bf e}_3^\top(\lCK \mathbf{R}(\psi) {^{\mathcal{K}}}\mathbf{p}_i^o + \lCK\mathbf{t})$: 
\begin{equation}\label{eq:linear_model_intermediate}
    {\bf e}_3^\top(\lCK \mathbf{R}(\psi) {^{\mathcal{K}}}\mathbf{p}_i^o + \lCK\mathbf{t}) \mathbf{q}_i = \mathbf{E}(\lCK \mathbf{R}(\psi) {^{\mathcal{K}}}\mathbf{p}_i^o + \lCK\mathbf{t}).
\end{equation}
We now define the state vector to be estimated as $\mathbf{x} \triangleq [\cos(\psi), \sin(\psi), \lCK\mathbf{t}^\top]^\top$. Let $\boldsymbol{\rho}_i \triangleq \mathbf{R}_{\theta\phi} {^{\mathcal{K}}}\mathbf{p}_i$ be the pre-rotated point and $\boldsymbol{\epsilon}_{{\bm \rho}_i} \triangleq \mathbf{R}_{\theta\phi}\boldsymbol{\epsilon}_{\mathbf{p}_i}$ be its propagated noise. By substituting ${^{\mathcal{K}}}\mathbf{p}_i^o = {^{\mathcal{K}}}\mathbf{p}_i - \boldsymbol{\epsilon}_{\mathbf{p}_i}$ and rearranging equation~\eqref{eq:linear_model_intermediate}, we obtain a linear equation with respect to $\mathbf{x}$:
\begin{equation} \label{single_linear_model}
    \mathbf{A}_i \mathbf{x} = \mathbf{b}_i + \boldsymbol{\epsilon}_i,
\end{equation}
where $\mathbf{b}_i = {\bm \rho}_i(3) {\bm q}_i$,
$\mathbf{A}_i=\left[\begin{smallmatrix}
   {\bm \rho}_i(1) & -{\bm \rho}_i(2) & 1 & 0 & -{\bm q}_i(1) \\
        {\bm \rho}_i(2) &{\bm \rho}_i(1) & 0 & 1 & -{\bm q}_i(2)
\end{smallmatrix}\right]$, and $\boldsymbol{\epsilon}_i=\left[\begin{smallmatrix}
  {\bm \epsilon}_{{\bm \rho}_i}(1) \cos(\psi) - {\bm \epsilon}_{{\bm \rho}_i}(2) \sin(\psi) - {\bm \epsilon}_{{\bm \rho}_i}(3){\bm q}_i(1) \\
        {\bm \epsilon}_{{\bm \rho}_i}(2) \cos(\psi) + {\bm \epsilon}_{{\bm \rho}_i}(1) \sin(\psi) - {\bm \epsilon}_{{\bm \rho}_i}(3){\bm q}_i(2)
\end{smallmatrix}\right]$.
Stacking~\eqref{single_linear_model} for all point correspondences yields the full linear system
\begin{equation} \label{full_linear_model}
    \mathbf{A}\mathbf{x} = \mathbf{b} + \boldsymbol{\epsilon}.
\end{equation}

\subsubsection{Bias-Eliminated Consistent Estimator}
Based on~\eqref{full_linear_model}, an ordinary least-squares (LS) solution can be obtained as
\begin{equation} \label{biased_LS}
    \hat{\mathbf{x}}^\mathrm{LS} = (\mathbf{A}^\top\mathbf{A})^{-1}\mathbf{A}^\top\mathbf{b}.
\end{equation}
However, the regressor term $\bf A$ contains noise and is correlated with the noise in the regressand term $\mathbf{b}$. As a result, this ordinary LS solution is generally statistically biased~\cite{zeng2024consistent}. To obtain a consistent estimate, we employ a BE approach that corrects for this correlation. The BE estimator is given by:
\begin{equation} \label{eqn:consistent_est}
    \hat{\mathbf{x}}^\mathrm{BE} = \left(\frac{\mathbf{A}^\top\mathbf{A}}{n} - \mathbf{G}_1\right)^{-1} \left(\frac{\mathbf{A}^\top\mathbf{b}}{n} - \mathbf{G}_2\right),
\end{equation}
where the bias correction terms $\mathbf{G}_1$ and $\mathbf{G}_2$ are derived from the known covariances ${\Sigma}_{\boldsymbol{\rho}_i} = \mathbf{R}_{\theta\phi}\boldsymbol{\Sigma}_{\mathbf{p}_i}\mathbf{R}_{\theta\phi}^{\top}$. Specifically, let $\boldsymbol{\Delta}\mathbf{A} = \mathbf{A} - \mathbf{A}^{o}$ and $\boldsymbol{\Delta}\mathbf{b}= \mathbf{b}- \mathbf{b}^{o}$. The first correction term $\mat{G}_1$ is given by
\begin{align*}
    \mat{G}_1 &= \mathbb{E}\left[\frac{1}{n}\boldsymbol{\Delta}\mat{A}^\top\boldsymbol{\Delta}\mat{A}\right] \\
    &= \left( \bar{\boldsymbol{\Sigma}}_{\vect{\rho}}(1,1) + \bar{\boldsymbol{\Sigma}}_{\vect{\rho}}(2,2) \right) \left( \mat{e}_1\mat{e}_1^\top + \mat{e}_2\mat{e}_2^\top \right),
\end{align*}
where $\bar{\boldsymbol{\Sigma}}_{\vect{\rho}} = \frac{1}{n}\sum_{i=1}^{n} \boldsymbol{\Sigma}_{\vect{\rho}_i}$ is the average covariance matrix. The second correction term, $\mat{G}_2 = \mathbb{E}[\frac{1}{n}\boldsymbol{\Delta}\mat{A}^\top\boldsymbol{\Delta}\mat{b}]$, is a five-dimensional vector whose last three components are zero. Its first two components are given by  $\frac{1}{n}\sum_{i=1}^{n} (\boldsymbol{\Sigma}_{\vect{\rho}_i}(1,3)\mat{I}_2 + \boldsymbol{\Sigma}_{\vect{\rho}_i}(2,3)\mat{J})[\vect{q}_i(1), \vect{q}_i(2)]^\top$, where $\mat{J}$ is the skew-symmetric matrix $\left[\begin{smallmatrix} 0 & 1 \\ -1 & 0 \end{smallmatrix}\right]$.


\begin{definition}[Big $O_p$]
    An estimator $\hat{\bm \gamma}$ is called a $\sqrt{n}$-consistent estimator of ${\bm \gamma}^o$ if $\hat{\bm \gamma}-{\bm \gamma}^o=O_p(1/\sqrt{n})$. That is, for any $\varepsilon >0$, there exist finite constants $M$ and $N$ such that for all $n>N$, $\mathbb{P} (\|\sqrt{n}(\hat {\bm\gamma} -{\bm\gamma}^o)\|>M )<\varepsilon$.
\end{definition}
        
\begin{theorem}\label{theorem_BE}
    The BE estimator $\hat{\bf x}^\mathrm{BE}$ is a $\sqrt{n}$-consistent estimator for the true state vector ${\bf x}^o$.
\end{theorem}
\begin{proof}
     The proof is based mainly on the following lemma: \begin{lemma}{(\cite{EIVPNPRAL2025}, Lemma 1):}\label{Lemma_1}
            Let $\{X_i\}$ be a sequence of independent random variables with $\mathbb{E}[X_i]=0$ and bounded $\mathbb{E}[X_i^2]$. Then, $\sum_{i=1}^n X_i/n = O_p(1/\sqrt{n})$.
		\end{lemma}
        
    In the noise-free case, $(\mathbf{A}^{o\top} \mathbf{A}^o)^{-1} \mathbf{A}^{o\top} \mathbf{b}^o$, or equivalently $(\frac{\mathbf{A}^{o\top} \mathbf{A}^o}{n})^{-1} \frac{\mathbf{A}^{o\top} \mathbf{b}^o}{n}$ results in the ground truth ${\mathbf{x}}^o$. The idea of the proof is to show that $\frac{\mathbf{A}^{\top} \mathbf{A}}{n}-\mathbf{G}_1$ and $\frac{\mathbf{A}^{\top} \mathbf{b}}{n} - \mathbf{G}_2$ converge to $\frac{\mathbf{A}^{o\top} \mathbf{A}^o}{n}$ and $\frac{\mathbf{A}^{o\top} \mathbf{b}^o}{n}$, respectively.  
    First, $\Delta \mathbf{A}^\top \Delta \mathbf{A}$ contains both the first and second order terms of $\boldsymbol{\epsilon}_{{\bm \rho}_i}$. For the second-order term $\boldsymbol{\epsilon}_{{\bm \rho}_i} \boldsymbol{\epsilon}_{{\bm \rho}_i}^\top$, we can subtract it by $ \boldsymbol{\Sigma}_{{\bm \rho}_i} $ to achieve the zero mean. It can be verified that $-\mathbf{G}_1$ actually performs this procedure. Hence, according to Lemma~\ref{Lemma_1}, it holds that 
    \begin{equation} \label{eqn:Op2}
             \frac{\mathbf{A}^{\top} \mathbf{A}}{n}-\frac{\mathbf{A}^{o\top} \mathbf{A}^{o}}{n}  =\frac{\Delta \mathbf{A}^{\top} \Delta \mathbf{A}}{n}+O_p(\frac{1}{\sqrt{n}}) \\
         = \mathbf{G}_1 + O_p(\frac{1}{\sqrt{n}}).
    \end{equation} 
    Second, similarly, for ${\bf A}^\top{\bf b}/n$, based on Lemma~\ref{Lemma_1}, we can obtain  
    \begin{equation} \label{eqn:Op1}
        \frac{\mathbf{A}^{\top} \mathbf{b}}{n}-\frac{\mathbf{A}^{o\top} \mathbf{b}^{o}}{n}= \mathbf{G}_2 + O_p(\frac{1}{\sqrt{n}}).
    \end{equation}
    Finally, by combining~\eqref{eqn:Op1} and~\eqref{eqn:Op2}, we obtain
    \begin{align*}
        \hat {\mathbf{x}}^{\rm BE} & = \left(\frac{\mathbf{A}^{o\top} \mathbf{A}^o}{n}+O_p(\frac{1}{\sqrt{n}})\right)^{-1} \left(\frac{\mathbf{A}^{o\top} \mathbf{b}^{o}}{n} +O_p(\frac{1}{\sqrt{n}})\right) \\
        & = \left(\frac{\mathbf{A}^{o\top} \mathbf{A}^o}{n}\right)^{-1} \frac{\mathbf{A}^{o\top} \mathbf{b}^{o}}{n} +O_p(\frac{1}{\sqrt{n}}) = {\bf x}^o + O_p(\frac{1}{\sqrt{n}}),
    \end{align*}
    which completes the proof.
\end{proof}

\subsection{Epipolar-Based Gauss-Newton Refinement}\label{subsec:GN_refine}
The previous stage provides a robust and consistent pose estimate. Using the consistent estimate as the initial value, we further refine the pose by implementing the  Gauss-Newton (GN) algorithm over the epipolar ML problem~\eqref{eq:4dof_objective} using the inlier set $\mathcal I$. Thanks to the $\sqrt{n}$-consistent initial estimator (Theorem~\ref{theorem_BE}), only a one-step GN iteration is sufficient to achieve the asymptotic efficiency, which is formally stated in the following lemma.
Denote the one-step GN iteration of $\hat {\bm \chi}^{\rm BE}$ by $\hat {\bm \chi}^{\rm GN}$.
\begin{lemma}[{\cite[Theorem 2]{mu2017globally}}]\label{GN_Consistancy}
    We have\begin{equation*}
		\hat {\bm \chi}^{\rm GN}-\hat {\bm \chi}^{\rm ML}=o_p(\frac{1}{\sqrt{n}}).
	\end{equation*}
\end{lemma}
The small $o_p$ implies that $\hat {\bm \chi}^{\rm GN}$ has the same asymptotic property as $\hat {\bm \chi}^{\rm ML}$~\cite{mu2017globally}, i.e., when the point number is large, $\hat {\bm \chi}^{\rm GN}$ achieves the theoretical lower bound--CRLB.

\section{Simultaneous Covariance Estimation and Pose Refinement} \label{jointly_est_refine}
Recall that we use IMU preintegration and visual PnP for relative 2-DOF (pitch and roll) and 4-DOF (yaw and translation) estimation, respectively. Without the gravity prior, the pitch and roll inevitably drift just like the remaining 4-DOF components of the pose. In this section, we further fuse visual information with the gravity prior to refine the 6-DOF pose and prevent the drift of pitch and roll angles. To handle time-varying accelerations and adaptively weight the gravity prior, we follow \cite{khosoussi2025joint} and jointly optimize the pose and the gravity prior covariance online.

The state vector is expanded to include not only the relative 4-DOF pose between the current frame and keyframe but also the absolute orientation at each intermediate IMU measurement time. Specifically, the state is defined as $\boldsymbol{\chi} = [\psi, \lCK\mathbf{t}^\top, \phi_1, \theta_1,\ldots,\phi_K, \theta_K]^\top$. Here, $\psi$ and $\lCK\mathbf{t}$ are the relative yaw and translation from the previous keyframe, while $({\phi_k, \theta_k})_{k=1}^K$ represents the absolute roll and pitch of the IMU with respect to the world frame $\mathcal{W}$ at each of the $K$ intermediate IMU measurement instants between the keyframe and current frame. The corresponding residuals are:

    \textbf{1) Visual residuals:} For each inlier $i \in \mathcal{I}$, this is the 2D reprojection error $\mathbf{r}_{\text{vis},i} = \mathbf{q}_i - h(\lCK\mathbf{R} {}^{\mathcal{K}}\mathbf{p}_i + \lCK\mathbf{t})$. Here, the relative rotation $\lCK\mathbf{R}$ is fully determined by $\{\psi,\phi_1,\theta_1,\phi_K,\theta_K,\lCI\mathbf{R}\}$.
  
    \textbf{2) Gravity prior residuals:} Under the assumption of low motion acceleration, the accelerometer reading $\bar{\mathbf{a}}_k$ should align with the gravity vector ${}^{{\mathcal{W}}}\mathbf{g}$ transformed into the IMU frame. We define the residual as the directional difference: $\mathbf{r}_{\text{imu},k}(\boldsymbol{\chi}) = \bar{\mathbf{a}}_k - {}^{\mathcal{I}}_{\mathcal{W}}\mathbf{R}(\phi_k, \theta_k) {}^{{\mathcal{W}}}\mathbf{g}$. Here, ${}^{\mathcal{I}}_{\mathcal{W}}\mathbf{R}(\phi_k, \theta_k)$ is the rotation from world to IMU constructed from the state variables $\phi_k$ and $\theta_k$, and ${}^{{\mathcal{W}}}\mathbf{g}$ is the unit gravity vector. This residual measures the non-gravitational component of acceleration. A non-zero value indicates either the violation of the low motion acceleration assumption or a misalignment in the estimated orientation $\{\phi_k, \theta_k\}$.
    
The underlying estimation problem is heteroscedastic due to multi-sensor fusion with time-varying noise. We jointly estimate the state $\boldsymbol{\chi}$ and the IMU measurement covariance $\boldsymbol{\Sigma}_{\text{imu}}$ by solving:
\begin{equation} \label{eq:joint_opt}
    \underset{\boldsymbol{\chi}, \boldsymbol{\Sigma}_{\text{imu}}}{\text{min}} \sum_{i \in \mathcal{I}} \|\mathbf{r}_{\text{vis},i}(\boldsymbol{\chi})\|^2_{\boldsymbol{\Sigma}_{\text{vis}}} + \sum_{k=1}^{K} \|\mathbf{r}_{\text{imu},k}(\boldsymbol{\chi})\|^2_{\boldsymbol{\Sigma}_{\text{imu}}},
\end{equation}
{where $\|\mathbf{r}\|_{\boldsymbol{\Sigma}} \triangleq {({\mathbf{r}}^\top \boldsymbol{\Sigma}^{-1}{\mathbf{r}})^{\frac{1}{2}}}$}. The visual covariance matrix $\boldsymbol{\Sigma}_{\text{vis}}$ is computed once at the beginning from the consistent 2D noise variance estimate obtained in Section~\ref{subsec:consistent_init} and remains fixed, reflecting the stable characteristics of feature detection.

We employ a BCD strategy to solve \eqref{eq:joint_opt}, iterating between updating the state $\boldsymbol{\chi}$ and the IMU covariance $\boldsymbol{\Sigma}_{\text{imu}}$.

    \textbf{1) State Update:} With $\boldsymbol{\Sigma}_{\text{imu}}$ held fixed, the state vector $\boldsymbol{\chi}$ is updated by solving the non-linear least-squares problem~\eqref{eq:joint_opt} using a GN iteration.
    
    \textbf{2) Covariance Update:} With $\boldsymbol{\chi}$ held fixed, we first compute the sample covariance of the IMU residuals $\mathbf{S}_{\text{imu}} = \frac{1}{K}\sum_{k=1}^K \mathbf{r}_{\text{imu},k} \mathbf{r}_{\text{imu},k}^\top$. This matrix directly reflects the robot's recent motion dynamics. During aggressive maneuvers, the low-acceleration assumption is violated, leading to large residuals and a large $\mathbf{S}_{\text{imu}}$. We then update $\boldsymbol{\Sigma}_{\text{imu}}$ using the analytical Wishart-MAP solution from \cite[Thm. 1]{khosoussi2025joint}. This update is a Bayesian procedure that computes a posterior by optimally blending the prior covariance with the likelihood information $\mathbf{S}_{\text{imu}}$. This mechanism adaptively adjusts the weight of the IMU gravity constraint: a larger sample covariance $\mathbf{S}_{\text{imu}}$ results in a larger posterior covariance $\boldsymbol{\Sigma}_{\text{imu}}$, which in turn reduces the weight ($\boldsymbol{\Sigma}_{\text{imu}}^{-1}$) of the IMU prior term during optimization.

The above gravity-assisted joint optimization serves two critical functions. First, it refines the initial 4-DOF pose into a full, high-fidelity 6-DOF estimate. Second and more importantly, it provides a mechanism to counteract the drift of the roll and pitch estimates. By continuously anchoring the orientation to the stable, non-drifting global gravity vector, this step effectively acts as a probabilistic complementary filter, correcting the drift inherent in the gyroscope integration and helping the decoupling of pitch and roll in Section~\ref{section::imu_gravity_est}.

\begin{figure}[!t]
    \centering
    \includegraphics[width=1\linewidth]{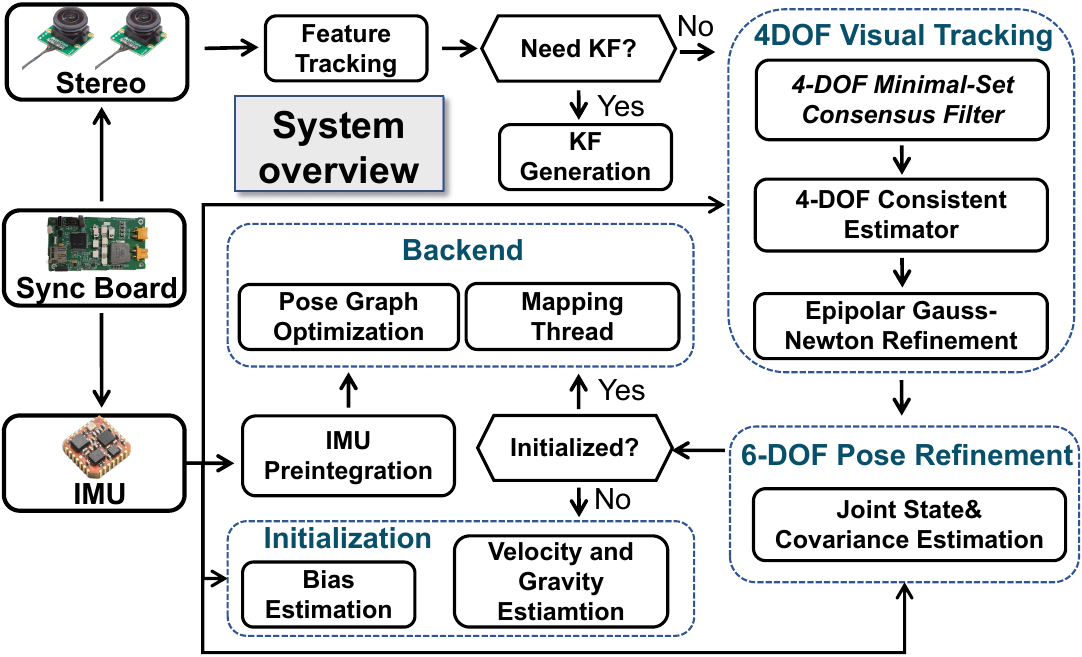}
    \caption{System architecture of the proposed stereo-inertial odometry framework.}
    \label{system_overview}
\end{figure}

\section{Sytem Realization}
In this section, we introduce the whole realization of our proposed GeVI-SLAM system.
The architecture of the GeVI-SLAM system, depicted in Fig.~\ref{system_overview}, includes a frontend and a backend to achieve both high-frequency pose tracking and long-term global consistency.

\subsection{Frontend}\label{section::frontend}

The frontend initializes, following VINS-Mono~\cite{qin2018vins}, by jointly estimating initial velocity, gravity direction, and gyroscope biases. With a stereo camera, the metric scale is directly observable, avoiding the large motions needed by monocular scale recovery and improving initialization robustness. For each new stereo frame, we track features via optical flow, reject outliers with a 4-DOF minimal-set consensus filter constrained by predicted IMU pitch and roll, run the 4-DOF consistent estimator~\eqref{eqn:consistent_est} followed by a single epipolar GN refinement, and finally use a joint state–covariance estimation module with the IMU gravity prior to output the pose. If too few features are tracked, the frame is promoted to a new keyframe, as in \cite{ferrera2021ov}.

\textbf{4-DOF Minimal-Set Consensus Filter.}
To efficiently identify the inlier set for pose tracking with pitch and roll constraints, the frontend embeds the minimal 3-point LS solver in Eq.~\eqref{biased_LS} into a RANSAC framework. The number of iterations needed to reach confidence $p$ with inlier ratio $w$ is $N \ge \log(1-p)/\log\bigl(1-w^{\,s}\bigr)$~\cite{fischler1981random}, where $s$ is the sample size per iteration. Note that to uniquely represent the 4-DOF pose $\boldsymbol{\chi}\in\mathbb{R}^4$, one needs a five-dimensional vector; we use $\mathbf{x} \triangleq [\cos(\psi), \sin(\psi), \lCK\mathbf{t}^\top]^\top \in \mathbb R^5$. Since each feature correspondence provides two independent constraints, $s=3$ (which we use) is the minimal case, offering a substantial efficiency advantage over conventional 5-point RANSAC ($s=5$). Assuming fixed $p$, the theoretical minimum iteration ratio between our 3-point and the 5-point schemes is $N_{3}/N_{5}=\Theta(w^{2})$. 

\subsection{Backend} \label{sec:backend}

The backend employs a dual-threaded optimization strategy to ensure both real-time performance and long-term global consistency. After the 6-DOF pose refinement in the frontend, the pose graph optimization thread optimizes all historical poses in an incremental manner, implemented with GTSAM \cite{gtsam}. 
The graph is connected by IMU preintegration factors and the relative pose factors derived from gravity-assisted visual estimation. To preserve real-time tractability, 3D landmarks are excluded in this optimization.
Concurrently, a mapping thread operates asynchronously to build a globally consistent map. Upon receiving a new keyframe, this thread triangulates new landmarks and initiates a local Bundle Adjustment (BA). Similar to \cite{ferrera2021ov}, this BA jointly optimizes the poses of a local window of keyframes and the 3D positions of their associated landmarks. Landmarks exhibiting large reprojection errors after optimization are deleted from the map. After the local BA converges, the optimized, more globally consistent poses are injected back into the pose graph as high-confidence absolute pose priors. These priors act as anchors, pulling the recent trajectory into alignment with the global map and thereby mitigating accumulated drift.



\section{Experiments} \label{experiments}
\begin{figure*}[t]
  \centering
  \subfloat[Yaw Error vs. Number of Points]{\includegraphics[width=0.24\textwidth]{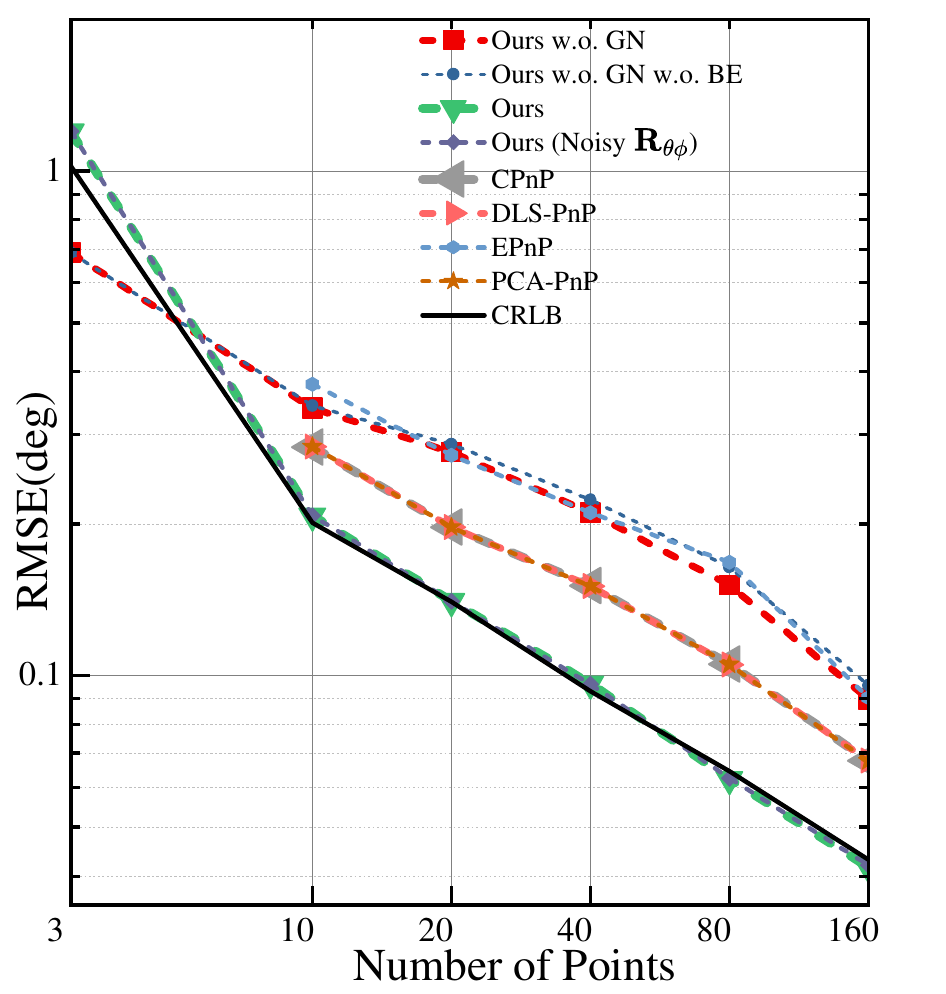}\label{fig:simu_yaw}}
  \hfill
  \subfloat[Translation Error vs. Number of Points]{\includegraphics[width=0.24\textwidth]{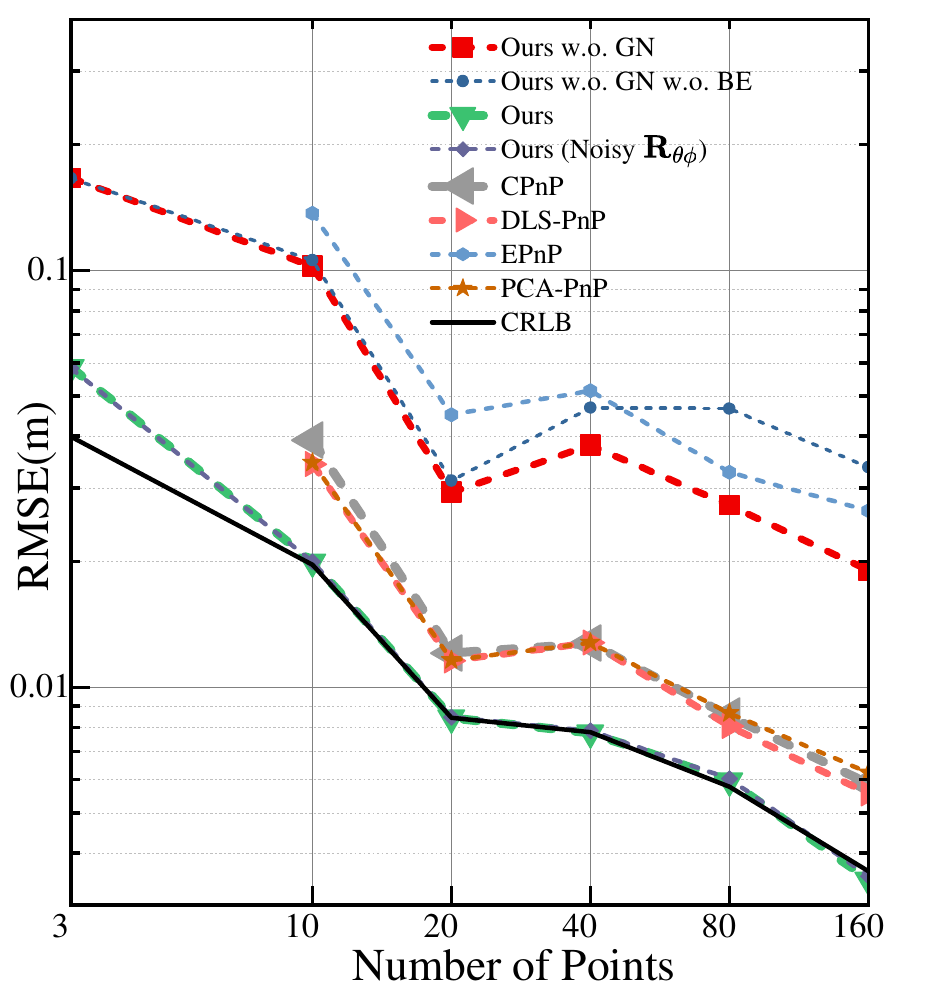}\label{fig:simu_t}}
  \hfill
  \subfloat[RANSAC Performance vs. Outlier Rate]{\includegraphics[width=0.24\textwidth]{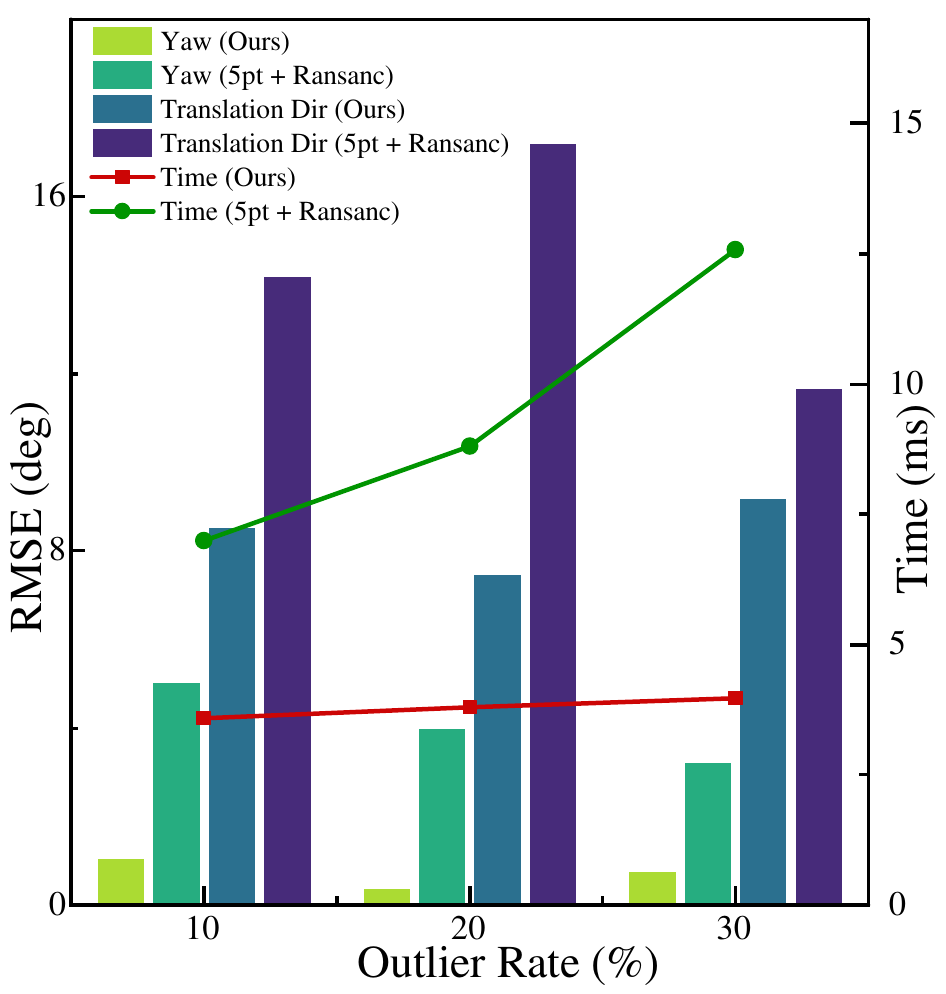}\label{fig:ransanc_simulation1}}
  \hfill
  \subfloat[Adaptive Fusion and Pose Error Drift]{\includegraphics[width=0.24\textwidth]{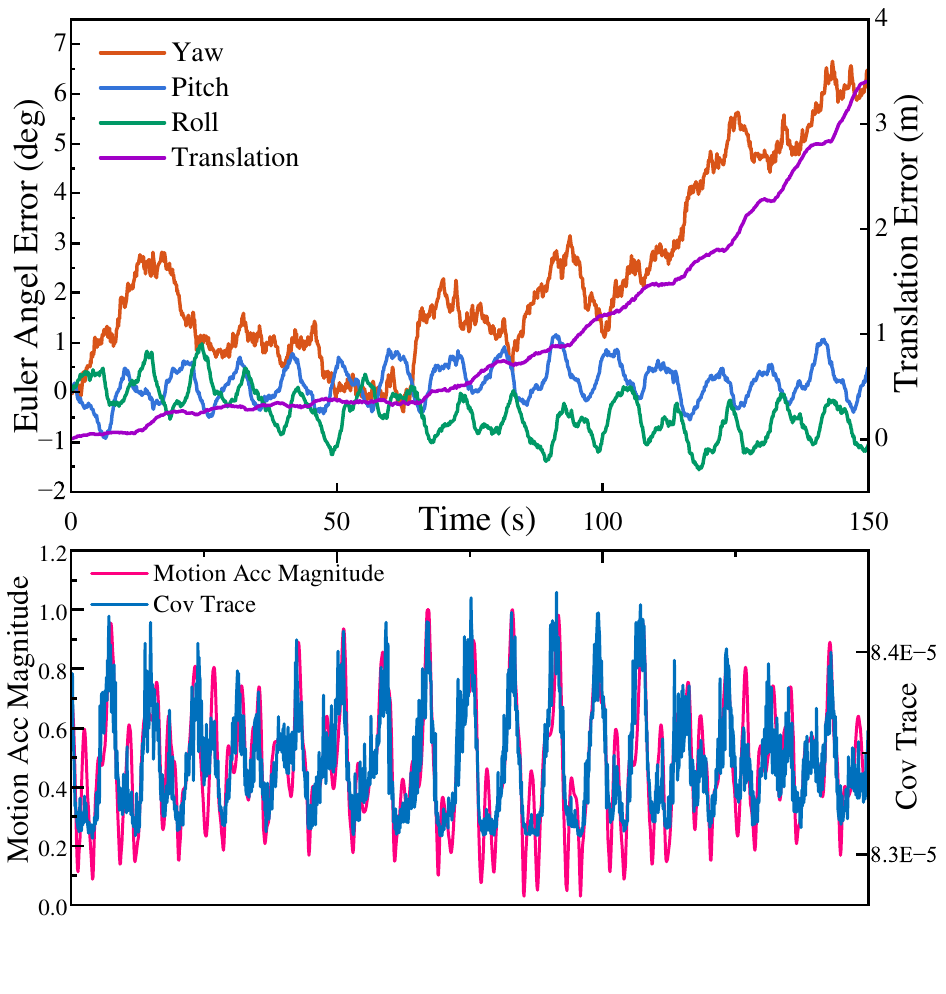}\label{fig:simu_joint}}
  \caption{\textbf{(a)-(b):} Our 4-DOF PnP estimator achieves CRLB accuracy, outperforming other PnP methods. 
  \textbf{(c):} Our 4-DOF minimal-set consensus filter achieves better accuracy at a lower computational cost than the 5-point method. 
  \textbf{(d):} Our adaptive fusion framework maintains drift-free roll and pitch estimates.}
  \label{fig:simulation_results}
\end{figure*}
\subsection{Experiments on Synthetic Data} \label{simulation}
The purposes of the simulations are three-fold: 1) verifying the consistency and asymptotic efficiency of our BE 4-DOF estimator; 2) demonstrating the efficiency and robustness of the proposed 4-DOF minimal-set consensus filter; 3) showing the ability to adaptively fuse the IMU gravity prior and maintain drift-free pitch and roll estimates.
In our simulations, the baseline of the stereo camera is set as ${\bf t}=[0.2,0,0]^\top~\text{m}$ and ${\bf R}=\mathbf{I}$. 
The camera has a focal length of $f=1100$ pixels, an image size of $800 \times 800$ pixels, and captures 3D points at depths ranging from 1 to 10 meters. 
Unless otherwise specified, Gaussian noises with a standard deviation of 2.5 pixels are added to 2D features.

\subsubsection{Estimator Consistency and Accuracy}
\label{subsubsec:sim_consistency}
We evaluate the statistical properties of our 4-DOF PnP estimator via a 700-run Monte Carlo simulation. The results, presented in Fig.~\ref{fig:simu_yaw} and Fig.~\ref{fig:simu_t}, benchmark the RMSE of our estimators against several widely used and state-of-the-art solvers, including CPnP\cite{zeng2023cpnp}, EPnP\cite{lepetit2009ep}, DLS-PnP\cite{hesch2011direct}, and PCA-PnP\cite{zhou2025novel}. The theoretical lower bound, CRLB, is also included. As depicted, our initial estimator (\texttt{Ours w.o. GN}) is consistent, as the RMSEs of both yaw and translation decrease linearly in the double-log plot as the point count increases. This property stems from our bias elimination method, which guarantees the estimate converges to the true pose. 
Leveraging this consistent estimate, a single-step GN iteration (\texttt{Ours}) is sufficient to asymptotically find the global minimum of the ML problem~\eqref{eq:4dof_objective} and achieve the CRLB accuracy. The compared methods are either inconsistent or unable to reach the CRLB. We also note that our 4-DOF estimator remains operational with as few as three points and consistently provides the most accurate estimates across all point numbers.
To further assess our estimator's robustness in a more practical context, we introduce roll and pitch noise with a standard deviation of 0.01$^{\circ}$ to simulate typical IMU pre-integration errors. Even with this perturbation, our method (\texttt{Ours (Noisy ${\bf R}_{\theta \phi}$)}) still coincides with the CRLB, highlighting the framework's resilience.



\subsubsection{Effectiveness of Minimal-Set Outlier Rejection}
\label{subsubsec:sim_ransac}
To evaluate practical robustness and efficiency with outliers, we vary the outlier ratio from 10\% to 30\% and run 400 Monte Carlo trials within a RANSAC framework. Our 4-DOF minimal-set consensus filter, in conjunction with the consistent estimator, uses a noisy roll and pitch prior with a $0.2^\circ$ standard deviation and is benchmarked against MATLAB’s 5-point RANSAC.
As shown in Fig.~\ref{fig:ransanc_simulation1}, our method delivers stable yaw and translation-direction estimates across all outlier rates and significantly outperforms the 5-point baseline.  Moreover, with 30\% outliers, our method reduces 68.5\% of the computation time compared to the 5-point essential-matrix RANSAC.
Together, these results highlight that our minimal-set consensus filter achieves both higher accuracy and lower computational complexity in practical scenarios with outliers.

\subsubsection{Validation of Adaptive Fusion for Drift-Free Pitch/Roll}
\label{subsubsec:sim_dynamics}
We simulate a 150-second trajectory, where the motion acceleration is time-varying but is upper-bounded by 1 m/s$^2$ by considering the water resistance. During the simulation, the gyroscope is corrupted by both white noise and a slowly drifting bias.
The pose error curves in Fig.~\ref{fig:simu_joint} show the expected yaw and translation drift, whereas the roll and pitch errors are drift-free and stay near zero, consistent with our claim on the gravity-enhanced joint optimization in Section~\ref{jointly_est_refine}. The figure also shows that as the motion acceleration increases, the estimated IMU gravity prior covariance grows, reducing the weight for the gravity prior residual. Hence, the estimator relies more on visual measurements.
In low-acceleration periods, the gravity prior regains influence, and the attitude is re-anchored, suppressing pitch and roll drift. These results confirm the effectiveness of our adaptive fusion for accurate and drift-free roll/pitch estimation over long horizons.

\subsection{Real Water Pool Experiments} \label{field_exp}
To validate real-world performance, we conducted underwater trials with a custom 60\,kg robot in a 10\,m $\times$ 5\,m, 2\,m-deep artificial pool with 10 underwater motion-capture cameras for ground truth. The robot carries a hardware-synchronized downward-facing stereo camera (1280$\times$720, 10\,Hz) and an MTI-630 IMU (200\,Hz).  
Two scenarios are covered: (i) scenes with numerous, highly similar feature points where we collect the \textbf{POOL-EASY} dataset and (ii) feature-sparse, largely coplanar scenes where we collect the \textbf{POOL-HARD} dataset. The \textbf{POOL-HARD} dataset features low texture and higher average accelerations.

We benchmark against state-of-the-art open-source methods: ORB-SLAM3~\cite{campos2021orb}, VINS-Fusion~\cite{qin2018vins}, and SVIN2~\cite{rahman2019svin2}, using ATE and RPE as metrics. For fairness, ORB-SLAM3 and VINS-Fusion are run with their EuRoC Micro Aerial Vehicle (MAV) dataset (EuRoC) configuration files, while SVIN2 is run in VI-only mode with sonar and depth sensors disabled. Qualitative results are shown in Fig.~\ref{trajectory}, which shows that our method achieves better consistency with ground truth.
 
The quantitative results in Table~\ref{tab:pool_metrics_singlecol} show superior accuracy and robustness of our method. 
On the \textbf{POOL-EASY} dataset, our method consistently achieves the best performance in both ATE and RPE, except for the ATE of Easy6 sequence. 
We note that SVIN2 fails to complete the full trajectory estimation in all sequences. This is because in these repetitive scenes, the BRISK feature descriptor used by SVIN2 yields a large portion of outliers that will be filtered out, leaving too few inliers to reliably initialize and track, as illustrated in the top right in Fig.~\ref{fig:orb_vs_brisk}.

The performance gap widens significantly on the more demanding \textbf{POOL-HARD} dataset. In this dataset, SVIN2 only fails in one sequence, as a calibration board is placed in the environment to aid initialization. However, ORB-SLAM3 fails in all sequences. This is because the ORB descriptor detects fewer stable corners in these texture-less scenes and is prone to early tracking loss, as shown in the bottom left in Fig.~\ref{fig:orb_vs_brisk}.
Compared to our methods, VINS-Fusion requires additional scale initialization, which needs more motion excitation and occasionally leads to failure.
Notably, our system is the only method that successfully completes all sequences on the \textbf{POOL-HARD} dataset.
Meanwhile, it exhibits significantly better RPE and consistently delivers the best or second-best ATE. The RPE superiority of our system is mainly attributed to the consistent 4-DOF PnP estimator and adaptive 6-DOF relative pose refinement.

\begin{figure}[h]
    \centering
    \includegraphics[width=0.8\linewidth]{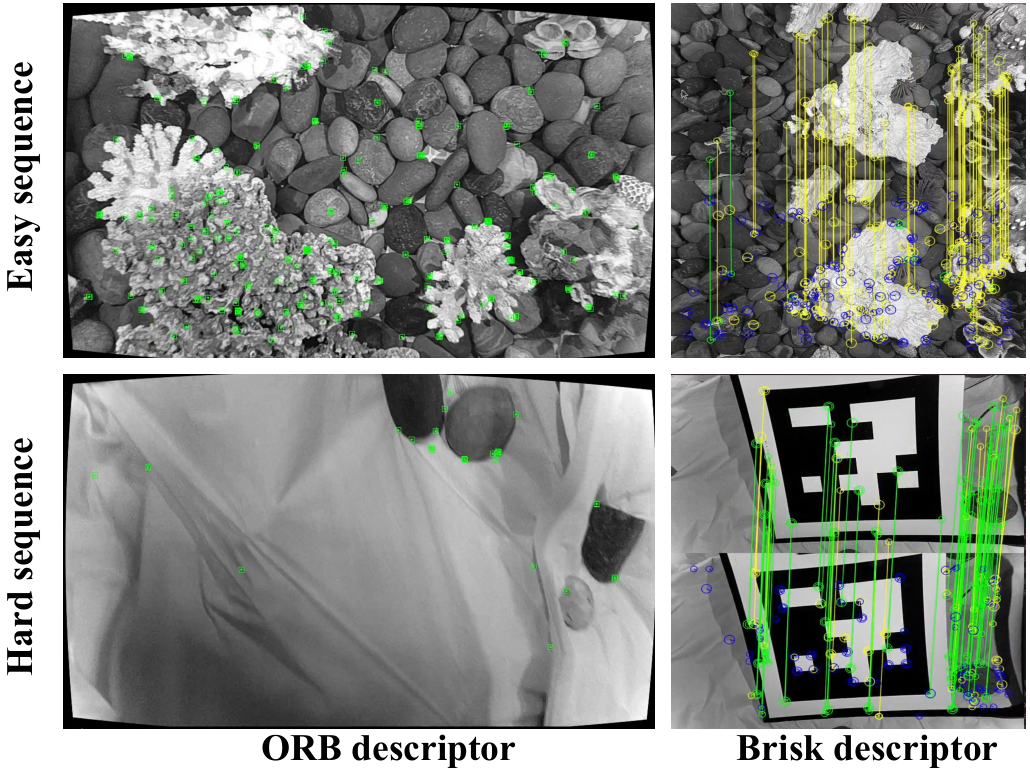}
    \caption{Illustration of feature matching: ORB-SLAM3 (ORB) vs. SVIN2 (BRISK). The green denotes inliers, while the yellow indicates outliers. }
    \label{fig:orb_vs_brisk}
\end{figure}

\begin{table}[t]
\centering
\footnotesize
\caption{Comparison of ATE and RPE on the \textbf{POOL-EASY} and \textbf{POOL-HARD} datasets. Values in {\color{blue}\textbf{blue bold}} are the smallest and {\color{blue!80}blue} are the second smallest.}
\label{tab:pool_metrics_singlecol}
\setlength{\tabcolsep}{2.5pt}
\begin{tabular}{l|cccc||cccc}
\toprule
\multirow{2}{*}{Sequence} & \multicolumn{4}{c||}{ATE (m)} & \multicolumn{4}{c}{RPE (m)} \\
\cmidrule(lr){2-5}\cmidrule(lr){6-9}
 & Ours & ORB* & VINS & SVIN2 & Ours & ORB* & VINS & SVIN2 \\
\midrule
\multicolumn{9}{c}{\textbf{POOL-EASY}} \\
\midrule
Easy1 & {\color{blue} \textbf{0.0116}} & 0.0128 & {\color{blue!80} 0.0127} & failed & {\color{blue} \textbf{0.0100}} & {\color{blue!80} 0.0101} & 0.0264 & failed \\
Easy2 & {\color{blue} \textbf{0.0173}} & 0.0196 & {\color{blue!80} 0.0190} & failed & {\color{blue} \textbf{0.0188}} & {\color{blue!80} 0.0188} & 0.0487 & failed \\
Easy3 & {\color{blue} \textbf{0.0319}} & {\color{blue!80} 0.0339} & failed & failed & {\color{blue} \textbf{0.0106}} & {\color{blue!80} 0.0106} & failed & failed \\
Easy4 & {\color{blue} \textbf{0.0211}} & {\color{blue!80} 0.0218} & 0.0657 & failed & {\color{blue} \textbf{0.0111}} & {\color{blue!80} 0.0111} & 0.0492 & failed \\
Easy5 & {\color{blue} \textbf{0.0147}} & {\color{blue!80} 0.0157} & 0.0158 & failed & {\color{blue} \textbf{0.0134}} & {\color{blue!80} 0.0135} & 0.0354 & failed \\
Easy6 & {\color{blue!80} 0.0219} & {\color{blue} \textbf{0.0217}} & 0.0408 & failed & {\color{blue} \textbf{0.0104}} & {\color{blue!80} 0.0105} & 0.0262 & failed \\
Easy7 & {\color{blue} \textbf{0.0270}} & 0.0485 & {\color{blue!80} 0.0344} & failed & {\color{blue} \textbf{0.0095}} & {\color{blue!80} 0.0095} & 0.0238 & failed \\
\midrule
\multicolumn{9}{c}{\textbf{POOL-HARD}} \\
\midrule
Hard1 & {\color{blue} \textbf{0.0246}} & failed & failed & {\color{blue!80} 0.0282} & {\color{blue} \textbf{0.0151}} & failed & failed & {\color{blue!80} 0.0631} \\
Hard2 & {\color{blue!80} 0.0463} & failed & 0.0513 & {\color{blue} \textbf{0.0312}} & {\color{blue} \textbf{0.0256}} & failed & {\color{blue!80} 0.0456} & 0.0710 \\
Hard3 & {\color{blue!80} 0.0406} & failed & 0.0353 & {\color{blue} \textbf{0.0343}} & {\color{blue} \textbf{0.0190}} & failed & {\color{blue!80} 0.0350} & 0.0640 \\
Hard4 & {\color{blue!80} 0.0486} & failed & 0.0538 & {\color{blue} \textbf{0.0301}} & {\color{blue} \textbf{0.0207}} & failed & {\color{blue!80} 0.0344} & 0.0731 \\
Hard5 & {\color{blue} \textbf{0.0333}} & failed & failed & failed & {\color{blue} \textbf{0.0158}} & failed & failed & failed \\
Hard6 & {\color{blue} \textbf{0.0280}} & failed & {\color{blue!80} 0.0338} & 0.0365 & {\color{blue} \textbf{0.0149}} & failed & {\color{blue!80} 0.0260} & 0.0598 \\
Hard7 & {\color{blue!80} 0.1054} & failed & failed & {\color{blue} \textbf{0.0536}} & {\color{blue} \textbf{0.0078}} & failed & failed & {\color{blue!80} 0.0485} \\
\bottomrule
\end{tabular}

\vspace{0.5em}
\parbox{\linewidth}{\footnotesize\raggedright * denotes the stereo-only mode. The stereo-IMU mode of ORB-SLAM3 is not used as it is difficult to successfully initialize underwater due to the robot's inability to provide sufficient acceleration for IMU excitation.}
\end{table}

\begin{figure*}[t]
    \centering
    \includegraphics[width=0.9\linewidth]{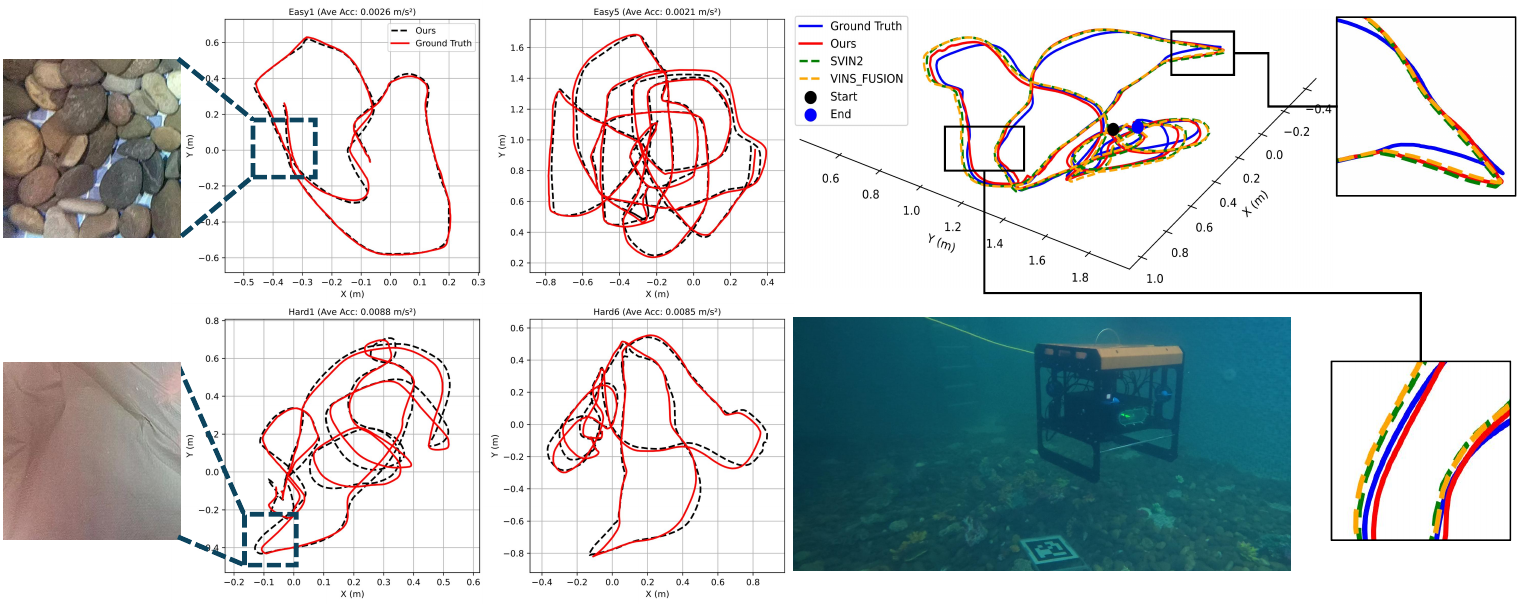}
    \caption{Trajectory evaluation on POOL-EASY/POOL-HARD. Left: two EASY and two HARD sequences with motion-capture ground truth and our estimates. Top right: a 3D overlay on an EASY sequence comparing GeVI-SLAM with SVIN2, ORB-SLAM3, and VINS-Fusion. Bottom right: the underwater robot during data collection.}
    \label{trajectory}
\end{figure*}

\section{Conclusion and Future Work}\label{section::conclusion}
We presented GeVI-SLAM, a gravity-enhanced stereo visual–inertial SLAM that reduces 6-DOF tracking to a 4-DOF PnP problem. A BE initializer with one-step Gauss–Newton achieves CRLB-level accuracy, and a 4-DOF minimal-set consensus filter efficiently rejects outliers. An adaptive 6-DOF refinement jointly estimates pose and IMU prior covariance, keeping roll/pitch drift-free. Simulations and real pool experiments show superior accuracy and robustness over ORB-SLAM3, VINS-Fusion, and SVIN2.

Limitations include yaw/translation drift without loop closure, the lack of refractive modeling, and uncharacterized roll/pitch uncertainty within the 4-DOF formulation. Future work will add robust loop closure and place recognition for underwater scenes, refractive camera modeling, fusion with DVL/pressure/range measurements, and field validation.









\bibliographystyle{IEEEtran}
\bibliography{bibfile}

\begin{thebibliography}{10}
\providecommand{\url}[1]{#1}
\csname url@samestyle\endcsname
\providecommand{\newblock}{\relax}
\providecommand{\bibinfo}[2]{#2}
\providecommand{\BIBentrySTDinterwordspacing}{\spaceskip=0pt\relax}
\providecommand{\BIBentryALTinterwordstretchfactor}{4}
\providecommand{\BIBentryALTinterwordspacing}{\spaceskip=\fontdimen2\font plus
\BIBentryALTinterwordstretchfactor\fontdimen3\font minus \fontdimen4\font\relax}
\providecommand{\BIBforeignlanguage}[2]{{%
\expandafter\ifx\csname l@#1\endcsname\relax
\typeout{** WARNING: IEEEtran.bst: No hyphenation pattern has been}%
\typeout{** loaded for the language `#1'. Using the pattern for}%
\typeout{** the default language instead.}%
\else
\language=\csname l@#1\endcsname
\fi
#2}}
\providecommand{\BIBdecl}{\relax}
\BIBdecl

\bibitem{mallios2017underwater}
A.~Mallios, E.~Vidal, R.~Campos, and M.~Carreras, ``Underwater caves sonar data set,'' \emph{The International Journal of Robotics Research}, vol.~36, no.~12, pp. 1247--1251, 2017.

\bibitem{joshi2019experimental}
B.~Joshi, S.~Rahman, M.~Kalaitzakis, B.~Cain, J.~Johnson, M.~Xanthidis, N.~Karapetyan, A.~Hernandez, A.~Q. Li, N.~Vitzilaios \emph{et~al.}, ``Experimental comparison of open source visual-inertial-based state estimation algorithms in the underwater domain,'' in \emph{2019 IEEE/RSJ International Conference on Intelligent Robots and Systems (IROS)}.\hskip 1em plus 0.5em minus 0.4em\relax IEEE, 2019, pp. 7227--7233.

\bibitem{xu2025aqua}
S.~Xu, K.~Zhang, and S.~Wang, ``Aqua-slam: Tightly-coupled underwater acoustic-visual-inertial slam with sensor calibration,'' \emph{IEEE Transactions on Robotics}, 2025.

\bibitem{rahman2019svin2}
S.~Rahman, A.~Q. Li, and I.~Rekleitis, ``{Svin2: An underwater slam system using sonar, visual, inertial, and depth sensor},'' in \emph{2019 IEEE/RSJ International Conference on Intelligent Robots and Systems (IROS)}.\hskip 1em plus 0.5em minus 0.4em\relax IEEE, 2019, pp. 1861--1868.

\bibitem{yao2024sg}
H.~Yao, Y.~Ma, P.~Li, C.~Zhai, J.~Song, M.~Ouyang, Z.~Dai, and X.~Zhu, ``Sg-vio: Monocular visual-inertial odometry with tightly coupled structural lines and gravity to avoid degeneracy,'' \emph{IEEE Internet of Things Journal}, 2024.

\bibitem{qin2018vins}
T.~Qin, P.~Li, and S.~Shen, ``Vins-mono: A robust and versatile monocular visual-inertial state estimator,'' \emph{IEEE transactions on robotics}, vol.~34, no.~4, pp. 1004--1020, 2018.

\bibitem{hu2021visual}
C.~Hu, S.~Zhu, Y.~Liang, Z.~Mu, and W.~Song, ``Visual-pressure fusion for underwater robot localization with online initialization,'' \emph{IEEE Robotics and Automation Letters}, vol.~6, no.~4, pp. 8426--8433, 2021.

\bibitem{campos2021orb}
C.~Campos, R.~Elvira, J.~J.~G. Rodr{\'\i}guez, J.~M. Montiel, and J.~D. Tard{\'o}s, ``Orb-slam3: An accurate open-source library for visual, visual--inertial, and multimap slam,'' \emph{IEEE Transactions on Robotics}, vol.~37, no.~6, pp. 1874--1890, 2021.

\bibitem{hildebrandt2010imu}
M.~Hildebrandt and F.~Kirchner, ``Imu-aided stereo visual odometry for ground-tracking auv applications,'' in \emph{OCEANS'10 IEEE SYDNEY}.\hskip 1em plus 0.5em minus 0.4em\relax IEEE, 2010, pp. 1--8.

\bibitem{bellavia2017selective}
F.~Bellavia, M.~Fanfani, and C.~Colombo, ``Selective visual odometry for accurate auv localization,'' \emph{Autonomous Robots}, vol.~41, no.~1, pp. 133--143, 2017.

\bibitem{yang2017feature}
N.~Yang, R.~Wang, and D.~Cremers, ``Feature-based or direct: An evaluation of monocular visual odometry,'' \emph{arXiv preprint arXiv:1705.04300}, pp. 1--12, 2017.

\bibitem{quattrini2016experimental}
A.~Quattrini~Li, A.~Coskun, S.~M. Doherty, S.~Ghasemlou, A.~S. Jagtap, M.~Modasshir, S.~Rahman, A.~Singh, M.~Xanthidis, J.~M. O’Kane \emph{et~al.}, ``Experimental comparison of open source vision-based state estimation algorithms,'' in \emph{International Symposium on Experimental Robotics}.\hskip 1em plus 0.5em minus 0.4em\relax Springer, 2016, pp. 775--786.

\bibitem{wang2023robust}
Y.~Wang, D.~Gu, X.~Ma, J.~Wang, and H.~Wang, ``Robust real-time auv self-localization based on stereo vision-inertia,'' \emph{IEEE Transactions on Vehicular Technology}, vol.~72, no.~6, pp. 7160--7170, 2023.

\bibitem{miao2021univio}
R.~Miao, J.~Qian, Y.~Song, R.~Ying, and P.~Liu, ``Univio: Unified direct and feature-based underwater stereo visual-inertial odometry,'' \emph{IEEE Transactions on Instrumentation and Measurement}, vol.~71, pp. 1--14, 2021.

\bibitem{huang2023tightly}
Y.~Huang, P.~Li, S.~Yan, Y.~Ou, Z.~Wu, M.~Tan, and J.~Yu, ``Tightly-coupled visual-dvl fusion for accurate localization of underwater robots,'' in \emph{2023 IEEE/RSJ International Conference on Intelligent Robots and Systems (IROS)}.\hskip 1em plus 0.5em minus 0.4em\relax IEEE, 2023, pp. 8090--8095.

\bibitem{lupton2011visual}
T.~Lupton and S.~Sukkarieh, ``Visual-inertial-aided navigation for high-dynamic motion in built environments without initial conditions,'' \emph{IEEE Transactions on Robotics}, vol.~28, no.~1, pp. 61--76, 2011.

\bibitem{li2019rapid}
J.~Li, H.~Bao, and G.~Zhang, ``Rapid and robust monocular visual-inertial initialization with gravity estimation via vertical edges,'' in \emph{2019 IEEE/RSJ International Conference on Intelligent Robots and Systems (IROS)}.\hskip 1em plus 0.5em minus 0.4em\relax IEEE, 2019, pp. 6230--6236.

\bibitem{usenko2016direct}
V.~Usenko, J.~Engel, J.~St{\"u}ckler, and D.~Cremers, ``Direct visual-inertial odometry with stereo cameras,'' in \emph{2016 IEEE international conference on robotics and automation (ICRA)}.\hskip 1em plus 0.5em minus 0.4em\relax IEEE, 2016, pp. 1885--1892.

\bibitem{noh2025garlio}
C.~Noh, W.~Yang, M.~Jung, S.~Jung, and A.~Kim, ``Garlio: Gravity enhanced radar-lidar-inertial odometry,'' \emph{arXiv preprint arXiv:2502.07703}, 2025.

\bibitem{kubelka2022gravity}
V.~Kubelka, M.~Vaidis, and F.~Pomerleau, ``Gravity-constrained point cloud registration,'' in \emph{2022 IEEE/RSJ international conference on intelligent robots and systems (IROS)}.\hskip 1em plus 0.5em minus 0.4em\relax IEEE, 2022, pp. 4873--4879.

\bibitem{mahony2008nonlinear}
R.~Mahony, T.~Hamel, and J.-M. Pflimlin, ``Nonlinear complementary filters on the special orthogonal group,'' \emph{IEEE Transactions on automatic control}, vol.~53, no.~5, pp. 1203--1218, 2008.

\bibitem{EIVPNPRAL2025}
G.~Zeng, Y.~Shen, Z.~Hong, Y.~Hong, V.~Ila, G.~Shi, and J.~Wu, ``Bias-eliminated pnp for stereo visual odometry: Provably consistent and large-scale localization,'' \emph{arXiv preprint arXiv:2504.17410}, 2025.

\bibitem{zeng2024consistent}
G.~Zeng, Q.~Zeng, X.~Li, B.~Mu, J.~Chen, L.~Shi, and J.~Wu, ``Consistent and asymptotically statistically-efficient solution to camera motion estimation,'' \emph{arXiv:2403.01174}, 2024.

\bibitem{mu2017globally}
B.~Mu, E.-W. Bai, W.~X. Zheng, and Q.~Zhu, ``A globally consistent nonlinear least squares estimator for identification of nonlinear rational systems,'' \emph{Automatica}, vol.~77, pp. 322--335, 2017.

\bibitem{khosoussi2025joint}
K.~Khosoussi and I.~Shames, ``Joint state and noise covariance estimation,'' \emph{arXiv preprint arXiv:2502.04584}, 2025.

\bibitem{ferrera2021ov}
M.~Ferrera, A.~Eudes, J.~Moras, M.~Sanfourche, and G.~Le~Besnerais, ``Ov$^{2}$slam: A fully online and versatile visual slam for real-time applications,'' \emph{IEEE Robotics and Automation Letters}, vol.~6, no.~2, pp. 1399--1406, 2021.

\bibitem{fischler1981random}
M.~A. Fischler and R.~C. Bolles, ``Random sample consensus: a paradigm for model fitting with applications to image analysis and automated cartography,'' \emph{Communications of the ACM}, vol.~24, no.~6, pp. 381--395, 1981.

\bibitem{gtsam}
\BIBentryALTinterwordspacing
F.~Dellaert and G.~Contributors, ``borglab/gtsam,'' May 2022. [Online]. Available: \url{https://github.com/borglab/gtsam)}
\BIBentrySTDinterwordspacing

\bibitem{zeng2023cpnp}
G.~Zeng, S.~Chen, B.~Mu, G.~Shi, and J.~Wu, ``Cpnp: Consistent pose estimator for perspective-n-point problem with bias elimination,'' in \emph{2023 IEEE International Conference on Robotics and Automation (ICRA)}.\hskip 1em plus 0.5em minus 0.4em\relax IEEE, 2023, pp. 1940--1946.

\bibitem{lepetit2009ep}
V.~Lepetit, F.~Moreno-Noguer, and P.~Fua, ``Ep n p: An accurate o (n) solution to the p n p problem,'' \emph{International journal of computer vision}, vol.~81, no.~2, pp. 155--166, 2009.

\bibitem{hesch2011direct}
J.~A. Hesch and S.~I. Roumeliotis, ``A direct least-squares (dls) method for pnp,'' in \emph{2011 International Conference on Computer Vision}.\hskip 1em plus 0.5em minus 0.4em\relax IEEE, 2011, pp. 383--390.

\bibitem{zhou2025novel}
L.~Zhou, Z.~Wei, and X.~Wang, ``A novel iterative solution to the perspective-$ n $-point problem via cost function approximation,'' \emph{IEEE Transactions on Robotics}, 2025.

\end{thebibliography}

\end{document}